\documentclass{article}

\usepackage{times}
\usepackage{fullpage}
\usepackage{booktabs} 
\usepackage{natbib}

\usepackage[utf8]{inputenc} 
\usepackage[T1]{fontenc}    
\usepackage{hyperref}       
\usepackage{url}            
\usepackage{booktabs}       
\usepackage{amsfonts}       
\usepackage{nicefrac}       
\usepackage{microtype}      
\usepackage{enumitem}
\usepackage{graphicx}

\usepackage{amsmath,amsfonts,amssymb,amsthm, color}
 \usepackage{algorithm,algorithmicx}
\usepackage{epsfig,wrapfig,epsf,subcaption}
\usepackage[outercaption]{sidecap} 
\usepackage{times}
\usepackage{fancyvrb}
\usepackage{enumerate}
\usepackage{tikz}
\usepackage{bm}
\usetikzlibrary{arrows,shapes}
\usetikzlibrary{patterns}
\usepackage{graphicx}
\usepackage{comment}
\usepackage{ltxtable}
\usepackage{nicefrac}
\usepackage{multirow}
\usepackage{setspace}
\usepackage[utf8]{inputenc} 
\usepackage[T1]{fontenc}    
\usepackage{url}            
\usepackage{booktabs}       
\usepackage{microtype}      
\usepackage{pbox}
\usepackage{xfrac}

\makeatletter
\newtheorem*{rep@theorem}{\rep@title}
\newcommand{\newreptheorem}[2]{%
\newenvironment{rep#1}[1]{%
 \def\rep@title{#2 \ref{##1}}%
 \begin{rep@theorem}}%
 {\end{rep@theorem}}}
\makeatother

\newtheorem{theorem}{Theorem}

\newreptheorem{theorem}{Theorem}
\newreptheorem{lemma}{Lemma}
\newreptheorem{definition}{Definition}

\theoremstyle{definition}

\newcounter{saveenumi}

\newcommand{\vecD}{\mathbf{D}}

\newcommand{\vecI}{\mathbf{I}}

\newcommand{\vecP}{\mathbf{P}}

\newcommand{\vecU}{\mathbf{U}}
\newcommand{\vecV}{\mathbf{V}}
\newcommand{\vecW}{\mathbf{W}}
\newcommand{\vecX}{\mathbf{X}}

\newcommand{\vecZ}{\mathbf{Z}}

\newcommand{\vece}{\mathbf{e}}

\newcommand{\vecv}{\mathbf{v}}

\newcommand{\removed}[1]{}

\newcommand{\norm}[1]{\left\lVert{#1}\right\rVert}

\newcommand{\R}{\mathbb{R}}

\newcommand{\m}[1]{\mathcal{#1}}

\newcommand{\reals}{\mathbb{R}}

\newcommand{\BL}{\text{BERT}_{\text{LARGE}}}

\newcommand{\CHANGES}[1]{{#1}}

\usepackage{color}

\setlist[itemize]{leftmargin=0.3in}
\setlist[enumerate]{leftmargin=0.3in}

\title{Low-Rank Bottleneck in Multi-head Attention Models}

\author{
Srinadh Bhojanapalli \thanks{
Google Research NY. \texttt{bsrinadh@google.com}}
\and
Chulhee Yun\thanks{ MIT. Based on work performed at Google Research New York. \texttt{chulheey@mit.edu} }
\and
Ankit Singh Rawat\thanks{Google Research NY. \texttt{ankitsrawat@google.com}}
\and
Sashank J. Reddi\thanks{Google Research NY. \texttt{sashank@google.com}}
\and
Sanjiv Kumar\thanks{Google Research NY. \texttt{sanjivk@google.com}}
}

\date{}
\begin{document}

\maketitle

\begin{abstract}
Attention based Transformer architecture has enabled significant advances in the field of natural language processing. In addition to new pre-training techniques, recent improvements crucially rely on working with a relatively larger embedding dimension for tokens. Unfortunately, this leads to models that are prohibitively large to be employed in the downstream tasks. In this paper we identify one of the important factors contributing to the large embedding size requirement. In particular, our analysis highlights that the scaling between the number of heads and the size of each head in the current architecture gives rise to a low-rank bottleneck in attention heads, causing this limitation. We further validate this in our experiments. As a solution we propose to set the head size of an attention unit to input sequence length, and independent of the number of heads, resulting in multi-head attention layers with provably more expressive power. We empirically show that this allows us to train models with a relatively smaller embedding dimension and with better performance scaling.

\end{abstract}

\section{Introduction}
Attention based architectures, such as Transformers, have been effective for sequence modelling tasks such as machine translation \citep{gehring2017convolutional, vaswani2017attention}, question answering, sentence classification \citep{radford2018improving,devlin2018bert} and document generation \citep{liu2018generating}. These models have emerged as better alternatives to the recurrent models - RNNs \citep{sutskever2014sequence}, LSTMs \citep{hochreiter1997long} and GRUs \citep{cho2014properties}. This is mainly due to their feed forward structure, which removes the sequential processing bottleneck for sequence data, making them easier to train compared to the recurrent models. Self attention models also have found applications in vision \citep{wang2018non}, adversarial networks \citep{zhang2018self}, reinforcement learning \citep{zambaldi2018relational, li2017deep} and speech recognition \citep{chiu2018state}.

Recent advances in using the self attention models in natural language tasks have been made by first using a language modeling task to pre-train the models and then fine tuning the learned models on specific downstream tasks. \citet{radford2018improving} and \citet{devlin2018bert} used Transformers to pre-train a language model and showed that the fine tuned model outperforms LSTMs on many natural language understanding and question answering tasks. For example, BERT \citep{devlin2018bert}, a 24 layer transformer model, is shown to achieve the state of the art performance on several NLP tasks, including on the SQuAD dataset. These advances, in addition to novel pre-training tasks, relied on bigger models with a larger embedding size. BERT model uses an embedding size of 1024 \citep{devlin2018bert}; GPT-2 uses models with embedding size up to 1600 \citep{radford2019language}.

A single Transformer block consists of two key components: a multi-head self attention layer followed by a feed forward layer \citep{vaswani2017attention}. A single head in a multi-head attention layer, computes self attention between the tokens in the input sequence, which it then uses to compute a weighted average of embeddings for each token. 
\CHANGES{Each head projects the data into a lower dimensional subspace, and computes the self attention in this subspace. This projection size for each head is commonly referred to as the {\em head size}.

To keep the number of parameters fixed in the attention layer regardless of the number of heads, the prevalent heuristic is to scale the head size with 1/(number of heads). 
This heuristic was initially proposed in \citet{vaswani2017attention} and has become a de facto standard heuristic in multi-head attention models \citep{radford2018improving,devlin2018bert}.} However, increasing the number of heads decreases the head size, decreasing the expressive power of individual heads. We prove that reducing the head size to a value below the input sequence length harms the representation power of each head (see Theorem \ref{thm:representation}). This is because a smaller head size introduces a rank constraint on the projection matrices in each head, and limits their representation power. We indeed notice this effect in practice: while the performance improves with increasing the number of heads in the beginning \citep{devlin2018bert}, we notice a drop in the performance once the number of heads increases beyond a certain threshold, as seen in Table~\ref{table:1} and Fig.~\ref{fig:lm1b} (see also Table~4(A) in \citet{vaswani2017attention}). 

In order to avoid hurting the performance, the existing models allow for multiple heads by increasing the embedding size, which in turn increases the head size. However, larger embedding size, in addition to increasing the number of parameters, makes it expensive to use the model and the learned embeddings in downstream tasks, as the downstream model sizes scale with the embedding size of the tokens. For example, the inference time and memory required in retrieval tasks typically increases linearly with the embedding size.

\begin{table}[!t]
\centering
\begin{small}
\begin{tabular}{|c| c c c|} 
 \hline
 \# heads & 8 & 16 & 32 \\ [0.5ex] 
 \hline
  \# params & 336M  & 336M & 336M \\ [0.5ex] 
 \hline
 SQuAD - F1 & 90.89$\pm$0.15 & 90.61$\pm$0.14 & 90.45$\pm$0.08 \\ [0.5ex]
 SQuAD - EM & 84.1$\pm$0.34 &83.75$\pm$0.27 & 83.48$\pm$0.13 \\[0.5ex]
 MNLI & 85$\pm$0.2 & 84.5$\pm$0.4 & 84.4$\pm$0.2 \\ [1ex] 
 \hline
\end{tabular}
\vspace{1ex}
\caption{Performance of $\BL$ \citep{devlin2018bert}, a 24 layer Transformer with an embedding size of 1024, suffers with the increasing number of heads after 8 heads.}
\label{table:1}
\end{small}
\end{table}

\CHANGES{In this paper we propose setting the head size of attention units to input sequence length. While this is a simple hyper-parameter change in the Transformer architecture, we show that it is important to set this value appropriately to avoid the low-rank bottleneck (see Theorem \ref{thm:representation}), and to improve the representation power (see Theorem \ref{thm:comparison}). This fixed head size is also independent of both the number of heads and the embedding size of the model. This allows us to train models with a relatively smaller embedding size (hence fewer parameters) without affecting the head size. Another advantage of the fixed head size is that unlike the standard setting which requires the number of heads to be a factor of the embedding size, we are free to set an arbitrary number of heads as required for the task.}

Interestingly, we note that this simple yet novel approach of fixing the head size in multi-head Transformers results in empirically superior performance. We evaluate Transformers trained with this fixed head size on language modeling (LM1B dataset), natural language inference (MNLI dataset) and question answering tasks (SQuAD dataset). We show that fixing the head size allows us to train Transformers with a better performance scaling and smaller embedding size. We show that with the fixed head size Transformers trained with an embedding size of 512 can match the performance of the $\BL$\citep{devlin2018bert}, a Transformer with an embedding size of 1024 (see Fig. \ref{fig:bert}). We further present experimental results evaluating the effect of different choices of the head size and the embedding size in Section \ref{sec:experiments}.

\CHANGES{Our contributions in this paper lie in identifying and rigorously proving the low rank bottleneck in multi-head attention models, and showing that fixing the head size to input sequence length results in a strictly better model, both theoretically and empirically. The contributions of this paper are summarized below.}
\begin{list}{\textbullet}{\leftmargin=1.1em \itemindent=0em \itemsep=1pt}
    \item We analyze the representation power of the multi-head self attention layer and prove the low-rank bottleneck the head size places on the attention units (Theorem \ref{thm:representation}).
    \item \CHANGES{We propose to set the head size to input sequence length, and show that fixing the head size strictly improves the expressive power of the multi-head attention layers compared to the standard heuristic for setting the head size (Theorem \ref{thm:comparison}). This allows us to both increase the number of heads per layer and decrease the embedding size, without hurting the performance.  We develop a novel construction based approach to prove this result, which can potentially be useful in analyzing other variants of the Transformer architecture.}
    \item We experimentally show that with a fixed head size, Transformers can be trained with better performance scaling and a smaller embedding size on three standard NLP tasks.
\end{list}

 \subsection{Related Works}
Given the significance of self attention models, there has been work trying to both improve the performance and speed up the computation in Transformers. \citet{ott2018scaling} and \citet{you2019reducing} reduce precision and use large batch training to reduce the training time of the attention models. \citet{child2019generating} propose sparse self attention models to speed up the computation in the attention layer for long sequence data generation tasks. They show that these sparse attention models can be trained on tasks with sequence length greater than 10k without sacrificing the accuracy. \citet{dehghani2018universal} propose a depth recurrent Transformer network that reuses the parameters across layers. They show that this modification makes the Transformer networks Turing complete even with finite precision weights. \citet{yang2019xlnet} propose a new way to increase the effective sequence length that the Transformer attends to, by reusing the intermediate embeddings across sequences. They show that the modified architecture performs better on tasks that require computing context over longer sequence lengths. We note that most of these modifications rely on the multi-head self attention, the same building block of the Transformers. Our work is studying this basic multi-head attention layer, and suggesting a new way to set the head size, which can be easily applied along with any of the above architectural modifications. 

\citet{wu2019pay} propose to replace the self-attention layer with lightweight dynamic convolutions and show improved performance on machine translation and language modeling. Even though the resulting model has faster inference time, it still needs to use a large embedding size (1024), as big as the original attention models. We believe the techniques in this paper can be combined with these results to realize both smaller embedding size and faster inference time.

\citet{sun2019token} perform neural architecture search using evolutionary methods on sequence to sequence models and find an evolved transformer architecture, which in addition to multi-head attention units, has convolution filter and gated linear units. Our proposed modifications stay closer to Transformers in spirit and can be used as seed units for this architecture search.

\citet{yang2017breaking} have studied the effect of rank constraint caused by the small projection sizes in computing the softmax loss. The situation in self attention layers is a bit different. While the expressive power of each head reduces with the decreasing head size, at the same time we are increasing the number of heads, which can potentially negate this and increase the overall capacity of the layer. As we show in Theorem~\ref{thm:comparison}, the prevalent head size heuristic indeed limits the expressive power of the multi-head attention layer.

\citet{yun2019transformers} studied the representation power of Transformers and showed that they are universal approximators of sequence to sequence functions. However they do not study the low rank bottleneck caused by the prevalent head size heuristic and its connection to the embedding size.

\citet{voita2019analyzing, michel2019sixteen} study the importance of different heads in an attention layer. They observe that, during inference, many of the heads in each layer can be pruned away with a little effect on the prediction. However, they still need multiple heads during the training.

\citet{child2019generating,correia2019adaptively} impose sparsity structure on the attention layer during training to improve both interpretability and performance. Fixing the head size will in fact make it easier to learn such sparsity patterns, as a low rank constraint does not allow a head to express all possible sparsity patterns. Combining these techniques can hence potentially enable training of sparse attention models with a smaller embedding size.
 \section{Transformer Architecture and Analysis}
In this section, we present the Transformer architecture and analyze the representation power of the multi-head self attention, a key component of the Transformer block.

The input to a Transformer network is a sequence of $n$ tokens. Typically, each token is converted into a token embedding of dimension $d$ by an embedding layer. We let $\vecX \in \R^{d \times n}$ be the embedding matrix corresponding to the $n$ tokens in the input sequence.

\subsection{Single-Head Attention}
\label{sec:singlehead}
The Transformer block is a combination of a self attention layer followed by a feed forward layer \citep{vaswani2017attention}. Both layers have a skip connection and use Layer Normalization (LN) \citep{ba2016layer}. In particular, for token embeddings $\vecX$, the dot product attention is computed as follows.
\begin{align}
\label{eq:attn}
\text{Attention}(\vecX) &=   \vecW_v \vecX \cdot {\rm Softmax}\left[ \frac{(\vecW_k \vecX)^T (\vecW_q \vecX)}{\sqrt{d_k}} \right] =  \vecW_v \vecX \cdot \vecP.
\end{align} 
Here $\vecW_q \in \R^{d_q \times d}$, $\vecW_k \in \R^{d_k \times d}$ and $\vecW_v \in \R^{d_v \times d}$ represent the projection matrices associated with the query, key and value respectively in an attention unit \citep{vaswani2017attention}. For a single-head attention unit, we have $d_q = d_k =d_v =d$. In the dot-product attention (cf.~\eqref{eq:attn}), $\vecP$ aims to capture the context of the input for a given token based on the remaining tokens in the input sequence. Subsequently, the output of the attention layer takes the following form.
\begin{align}
\label{eq:attn_out}
\text{LN}\left( \vecX + \vecW_o \cdot \text{Attention}(\vecX) \right),
\end{align}
where $\text{LN}(\cdot)$ represents the layer-normalization operation. Given the attention module, as defined in \eqref{eq:attn}, it is natural to question its ability to represent arbitrary contexts $\vecP$ for a given input sequence $\vecX$.

In the following result we establish that for a large enough projection size an attention unit can represent any data pair $(\vecX, \vecP)$. We also show that the model cannot represent arbitrary context when $d$ is smaller than $n$, creating a low-rank bottleneck.
\begin{theorem}[Representation Theorem]\label{thm:representation}
If $d_q = d_k = d \geq n$, then 
given any full column rank matrix $\vecX \in \R^{d \times n}$ and an arbitrary $n \times n$ positive column stochastic matrix $\vecP$, there always exists $d\times d$ projection matrices $\vecW_q$ and $\vecW_k$ such that 
\begin{equation}\label{eq:representation}
  {\rm Softmax}\left[ \frac{(\vecW_k \vecX)^T (\vecW_q \vecX)}{\sqrt{d_k}} \right] = \vecP.  
\end{equation}
If $d_q = d_k = d < n$, there exist $\vecX$ and $\vecP$ such that \eqref{eq:representation} does not hold for all $\vecW_q$ and $\vecW_k$.
\end{theorem}

This result shows that the projection dimension $d_q =d_k = d$ needs to be larger than the sequence length $n$ for the attention unit to be able to represent any desired context $\vecP$. Even though this result describes a single example sequence case, 
\CHANGES{it highlights a fundamental property of the model architecture that decreasing the projection size below a certain threshold introduces a bottleneck.}

\begin{proof}[Proof of Theorem \ref{thm:representation}]
$\mathbf{d \geq n}$ \textbf{case}.
To prove the first part of the result, we present an explicit construction of $\vecW_k$ and $\vecW_q$ which allows us to generate $\vecP$ from $\vecX$ using the dot product attention. Since $\vecX$ has full column rank, there exists a left inverse $\vecX^\dagger = (\vecX^T \vecX)^{-1} \vecX^T \in \reals^{n \times d}$ such that $\vecX^\dagger \vecX = \vecI_n$. Let $\vecW_k = \tilde{\vecW}_k \vecX^\dagger$ and  $\vecW_q = \tilde{\vecW}_q \vecX^\dagger$. Then  
\begin{align}
\label{eq:step1}
    \vecX^T \vecW_k^T \vecW_q \vecX = \vecX^T (\vecX^\dagger)^T \tilde{\vecW}_k^T \tilde{\vecW}_q \vecX^\dagger \vecX 
    &= \vecI_n \cdot \tilde{\vecW}_k^T \tilde{\vecW}_q \cdot \vecI_n \nonumber \\
    &=\tilde{\vecW}_k^T \tilde{\vecW}_q = \tilde{\vecW}_{kq}.
\end{align}
Now that the above choice of $\vecW_q$ and $\vecW_k$ has handled the dependence on $\vecX$, we will choose a $\tilde{\vecW}_{kq}$ depending on $\vecP$ and finish the construction. Below we express the Softmax operation on the query and key inner products. Note that the Softmax here is a columnwise operator computing the attention scores for each query. By using \eqref{eq:step1}, we obtain that
\begin{align*}
{\rm Softmax}\left[ \frac{(\vecW_k \vecX)^T (\vecW_q \vecX)}{\sqrt{d_k}} \right] &= {\rm Softmax}\left[ \frac{\tilde{\vecW}_{kq}}{\sqrt{d_k}} \right] 
= \exp\left(\frac{\tilde{\vecW}_{kq}}{\sqrt{d_k}}\right) \cdot \vecD_{\tilde{\vecW}_{kq}}^{-1},
\end{align*}
where $\vecD_{\tilde{\vecW}_{kq}}$ is an $n\times n$ diagonal matrix such that \begin{align*}
    (\vecD_{\tilde{\vecW}_{kq}})_{ii} &= \sum_{j=1}^n \exp\left(\frac{(\tilde{\vecW}_{kq})_{ji}}{\sqrt{d_k}}\right) = \left(\mathbf{1}^T \exp\left(\frac{(\tilde{\vecW}_{kq})}{\sqrt{d_k}}\right)\right)_i .
\end{align*}

Hence, we can establish the desired result by showing that there always exists a $\tilde{\vecW}_{kq}$ that satisfies the following fixed point equation. 
\begin{equation}
\label{eq:step2}
      \exp\left(\frac{\tilde{\vecW}_{kq}}{\sqrt{d_k}}\right) =  \vecP \cdot \vecD_{\tilde{\vecW}_{kq}}.
\end{equation}

Given $\vecP$, to construct such a $\tilde{\vecW}_{kq}$, we pick an arbitrary positive diagonal matrix $\vecD_0$, and set 
\begin{align}
    \label{eq:step3}
     \tilde{\vecW}_{kq} =\sqrt{d_k} \cdot \log \left( \vecP \cdot \vecD_0 \right).
\end{align}
Since $\vecP$ is a positive matrix, such a $\tilde{\vecW}_{kq}$ always exists. Next, we verify that this construction indeed satisfies the fixed point equation (cf.~\eqref{eq:step2}). Note that
\begin{align}
\vecD_{\tilde{\vecW}_{kq}} = {\rm Diag}\left(\mathbf{1}^T \exp\left(\frac{(\tilde{\vecW}_{kq})}{\sqrt{d_k}}\right) \right) 
= {\rm Diag}\left(\mathbf{1}^T \vecP \cdot \vecD_0 \right)  = \vecD_0. \label{eq:step4}
\end{align}
The last 
equation follows from the fact that $\vecP$ is a column stochastic matrix. Now, using \eqref{eq:step3} and \eqref{eq:step4},
\begin{align*}
    \exp\left(\frac{\tilde{\vecW}_{kq}}{\sqrt{d_k}}\right) = \vecP \cdot \vecD_0 
    = \vecP \cdot \vecD_{\tilde{\vecW}_{kq}}.
\end{align*}
This completes the first part of the proof. \\ \ \\
\noindent $\mathbf{d < n}$ \textbf{case}. Consider the case of $d=1$ and $n=2$. Then $\vecX \in \R^{1 \times 2}$ and $\vecW_q$ and $\vecW_k \in \R^{1 \times 1}$. Let $\vecX =\lbrack 1, 0\rbrack$. Then 
\begin{align*}{\rm Softmax}&\left[ \frac{(\vecW_k \vecX)^T (\vecW_q \vecX)}{\sqrt{d_k}} \right] = {\rm Softmax}\left[ \frac{ \lbrack 1, 0\rbrack^T \vecW_k^T \vecW_q \lbrack 1, 0\rbrack}{\sqrt{d_k}} \right]= {\rm Softmax}\left[ \begin{bmatrix} \vecW_k \vecW_q & 0 \\ 0 & 0 \end{bmatrix} \right].\end{align*}
This matrix clearly cannot be used to generate $\vecP$ that have distinct elements in the second column, e.g., $\vecP=\begin{bmatrix} 0.5 & 0.75 \\ 0.5 & 0.25 \end{bmatrix}$.
\end{proof}


\subsection{Multi-Head Attention}
\label{sec:multihead}
As discussed in Section~\ref{sec:singlehead}, an attention unit updates the embedding of an input token based on a weighted average of the embeddings of all the tokens in the sequence, using the context $\vecP$ (cf.~\eqref{eq:attn}). \citet{vaswani2017attention} proposed Multi-Head attention mechanism that increases the representation power of an attention layer, where multiple attention units operate on different low dimensional projections of the input, with each attention unit being referred to as a head. This is followed by concatenation of the outputs from different heads. In particular, the computation inside a Multi-Head attention with $h$ heads takes the following form:
\begin{align*}
\text{head}(\vecX)_i  =\vecW_v^i  \vecX \cdot {\rm Softmax}\left[\nicefrac{(\vecW_k^i \vecX)^T (\vecW_q^i \vecX)}{\sqrt{\frac{d}{h}}} \right]  \in \R^{\frac{d}{h} \times n}
\end{align*}
\begin{align*}
\text{MultiHead}(\vecX) = \text{Concat}[\text{head}(\vecX)_1, \cdots, \text{head}(\vecX)_h] \in \R^{d \times n}.
\end{align*}
The output of the Multi-head attention layer then becomes
\begin{equation}
\label{eq:mutlihead}
\vecZ = \text{LN}\left( \vecX + \vecW_o \cdot \text{MultiHead}(\vecX) \right),
\end{equation}
where $\vecW_o \in \R^{d \times d}$. For a model with $h$ heads, the query, key and value projection matrices $\{\vecW_q^i\}$, $\{\vecW_k^i\}$ and $\{\vecW_v^i\}$ are $\frac{d}{h} \times d$ matrices. Therefore, each head projects the input onto a $\frac{d}{h}$-dimensional subspace to compute the context, and keeps the number of parameters fixed per layer. Using MultiHead has resulted in empirically better performance over the single head attention layer \citep{vaswani2017attention}.

\subsection{Low-Rank Bottleneck}
While increasing the number of heads seemingly gives the model more expressive power, at the same time we are reducing the head size, which can decrease the expressive power. When the number of heads $h$ is larger than $\frac{d}{n}$, the attention unit inside each head projects onto a dimension smaller than $n$, creating a low-rank bottlenck and loses its ability to represent arbitrary context vectors (cf.~Theorem \ref{thm:representation}). 
Interestingly, this is consistent with the empirical observation in Table~\ref{table:1} that increasing $h$ beyond 8 results in performance degradation in $\BL$ \citep{devlin2018bert}; note that $d = 1024$ and $n = 128$ for most of the pre-training phase of $\BL$.

Since the sequence length is fixed from the data/task at hand, the only way to increase the number of heads without introducing the low-rank bottleneck is by increasing the embedding size $d$. This is a fundamental limitation of the currently dominant head size heuristic, that we need to increase the embedding size in order to support more heads.

Unfortunately, increasing the embedding size leads to higher computation and memory requirements to train and store the model. Further, since it is common to use learned embeddings from Transformer based models for downstream tasks \citep{devlin2018bert}, larger embedding size increases the model size and computation required for all the downstream tasks as well.

 \newcommand{\sfmx}{{\rm Softmax}}

\begin{figure*}[!t]
\vskip 0.2in
\begin{center}
\begin{subfigure}[b]{0.48\textwidth}
	\includegraphics[width=\columnwidth]{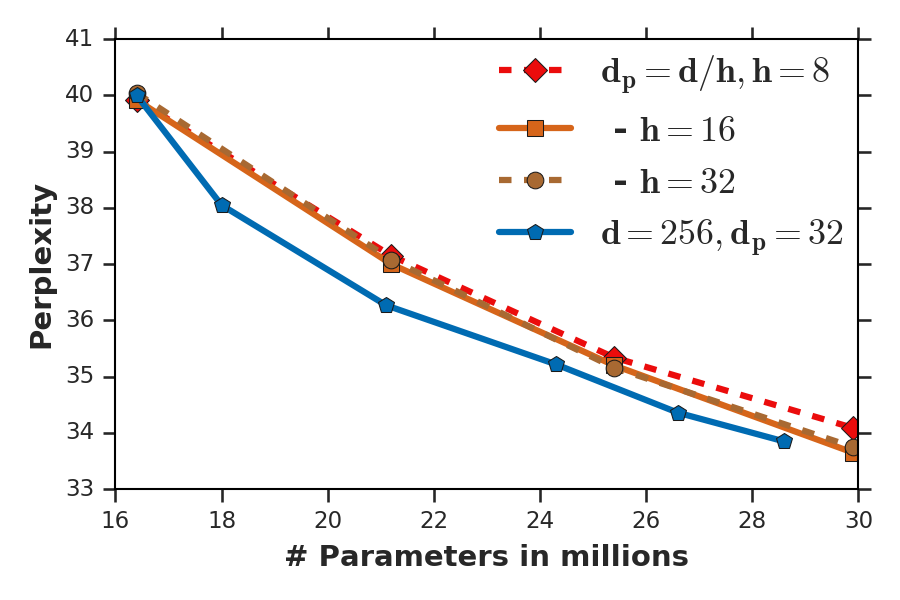}
	\caption{LM1B}
	\label{fig:lm1b_params}
	\end{subfigure}
	\begin{subfigure}[b]{0.48\textwidth}
	\includegraphics[width=\columnwidth]{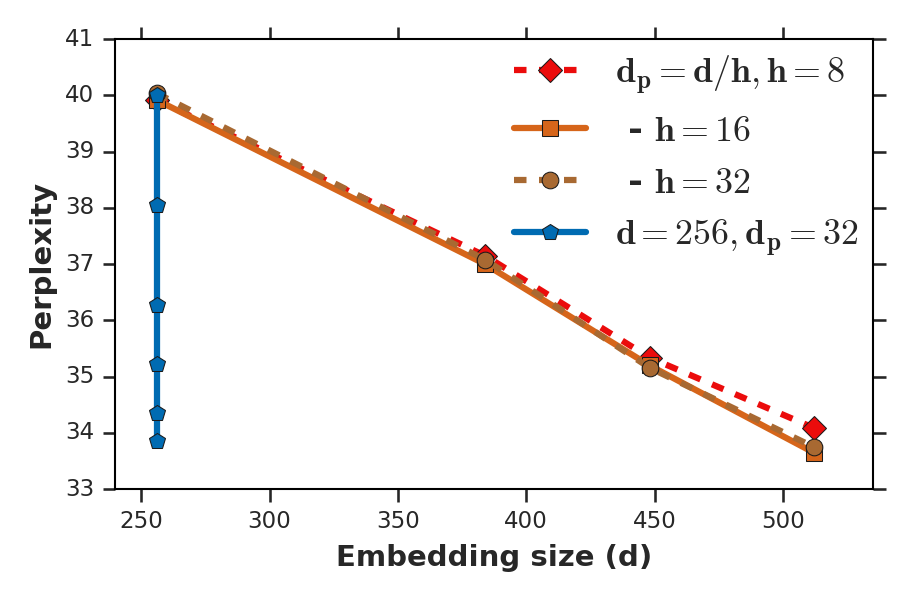}
	\caption{LM1B}
	\label{fig:lm1b_embed}
	\end{subfigure}
	\caption{ \CHANGES{Performance of Transformers trained with the prevalent head size heuristic ($d_p =\nicefrac{d}{h}$) (baseline) compared with the fixed head size ($d_p=32$) on a language modeling task (LM1B) on the test set. We train baseline models with embedding sizes from 256 to 512. We train the fixed head size models with a fixed embedding size of 256 and a head size of 32, and vary the number of heads from 4 to 70, while matching the number of parameters. The plots clearly indicate that fixing the head size allows us to train Transformers with a smaller embedding size (plot (b)), and with a better scaling of performance (plot (a)). Note that for perplexity lower values are better.}}
  \label{fig:lm1b}
  \end{center}
\vskip -0.2in
\end{figure*}

\section{Fixed Multi-Head Attention}
\label{sec:cmt}
In this section we propose to fix the head size of the Transformer, which allows us to enjoy the advantage of higher expressive power of multiple heads without requiring the embedding size to be large. The key is to decouple the dependency between the projection size in a head and the embedding size of the model. The projection matrices now project onto subspaces of a fixed dimension $d_p$ irrespective of the number of heads $h$. This approach where $d_p$ is independent of $d$ and $h$ leads to the following attention mechanism.
\begin{align*}
\text{\rm fixedhead}(\vecX)_i   =  \vecV_v^i  \vecX \cdot {\rm Softmax}\left[\nicefrac{(\vecV_k^i \vecX)^T (\vecV_q^i \vecX)}{\sqrt{d_p}} \right]  \in \R^{d_p \times n}
\end{align*}
\begin{align*}
\text{FixedMultiHead}(\vecX)  =  \text{Concat}[\text{\rm fixedhead}(\vecX)_1, \cdots, \text{\rm fixedhead}(\vecX)_h] \in \R^{d_p \cdot h \times n}.
\end{align*}
Note that the projection matrices used here $\{\vecV_q^i\}$, $\{\vecV_k^i\}$ and $\{\vecV_v^i\}$ are $d_p \times d$ matrices. 
With $\vecV_o \in \R^{d \times h\cdot d_p}$, the output of this new multi-head attention layer takes the following form. 
\begin{align*}
\vecZ = \text{LN}\left( \vecX + \vecV_o \cdot \text{FixedMultiHead}(\vecX) \right).
\end{align*}

This modification makes each attention head more similar to a hidden unit in a feed forward network or a filter in a convolutional network, and allows us to vary the number of heads without worrying about reducing the representation power per head. The downside is, unlike the standard MultiHead, the number of parameters per layer increases with the number of heads. However, this modification allows us to train a model with a smaller embedding size without a low-rank bottleneck, ultimately allowing us to reduce the total number of parameters in the model.


\subsection{MultiHead vs.\ FixedMultiHead Attention}

Given a MultiHead layer, we can always represent it using a FixedMultiHead layer, whenever we have the head size $d_p \geq \nicefrac{d}{h}$.
While this shows that increasing the number of heads $h$ beyond $\nicefrac{d}{d_p}$ makes individual heads of the FixedMultiHead as expressive as the ones in the MultiHead, it is not obvious if FixedMultiHead is \emph{strictly} more expressive. Can the FixedMultiHead layer represent functions that the standard MultiHead layer can not represent? In this subsection we show that indeed, in the multi-head regime, the FixedMultiHead layer is strictly better than the standard MultiHead layer in terms of expressive power.

Consider the standard multi-head attention units in \eqref{eq:mutlihead}.\begin{equation*}
    f_\vecW (\vecX) = \vecW_o \cdot \text{MultiHead}(\vecX).
\end{equation*}
We denote the collection of all parameter matrices as $\vecW$. Similarly, consider the function represented by the fixed head size attention units:
\begin{equation*}
    g_\vecV (\vecX) =  \vecV_o \cdot \text{FixedMultiHead}(\vecX). 
\end{equation*} Let $\vecV$ be the collection of all these parameter matrices.
We define $\m F$ and $\m G$ to be the class of functions $f_\vecW(\cdot)$ and $g_\vecV(\cdot)$, respectively. As noted above, if $d_p \geq \nicefrac{d}{h}$, we have $\m F \subset \m G$.

The following theorem shows that even for simple examples in $\m G$, functions in $\m F$ fail to  
\CHANGES{represent them; this already shows that $\m F$ is a \emph{strict} subset of $\m G$.}
\begin{theorem}
\label{thm:comparison}
Let $n \geq 2$, $d \geq d_p$, and $h > \nicefrac{d}{d_p}$. Consider a \textup{FixedMultiHead} attention layer $g_\vecV(\cdot)$ with parameters that satisfy the following conditions: 
\begin{equation*}
\vecV_o \times \begin{bmatrix}\vecV_v^1 \\ \vdots \\ \vecV_v^h \end{bmatrix}
\text{ is full rank, and~}
(\vecV_k^i)^T \vecV_q^i = \vecU,  \text{ for all } i = 1, \dots, h,
\text{ where $\vecU$ is a rank-$d_p$ matrix. }
\end{equation*}
Then, for any $f_{\vecW} \in \m F$, there exists $\vecX \in \reals^{d \times n}$ such that $f_{\vecW}(\vecX) \neq g_{\vecV}(\vecX).$
\end{theorem}
Because $\norm{f_{\vecW}(\vecX)-g_{\vecV}(\vecX)}$ is a continuous function of $\vecX$, existence of such an $\vecX$ implies that the integral of the norm of difference (i.e., approximation error) is strictly positive. 
\CHANGES{We note that the assumptions on $\vecV_k^i$ and $\vecV_q^i$ in the above Theorem are made to provide a simple and constructive proof; in fact, failure of MultiHead ($\m F$) to represent such simple attention layers suggests that the situation is likely worse for more complex functions.}

Theorem~\ref{thm:comparison} shows that the expressive power of the FixedMultiHead attention function class is strictly superior to the standard MultiHead attention function class. Hence the heuristic of reducing the head size with the number of heads is limiting the expressive power of MultiHead, whereas using the fixed head size will increase the expressive power of the attention layers.

\begin{figure*}[!t]
\centering
    \begin{subfigure}[b]{0.32\textwidth}
	\includegraphics[width=\textwidth]{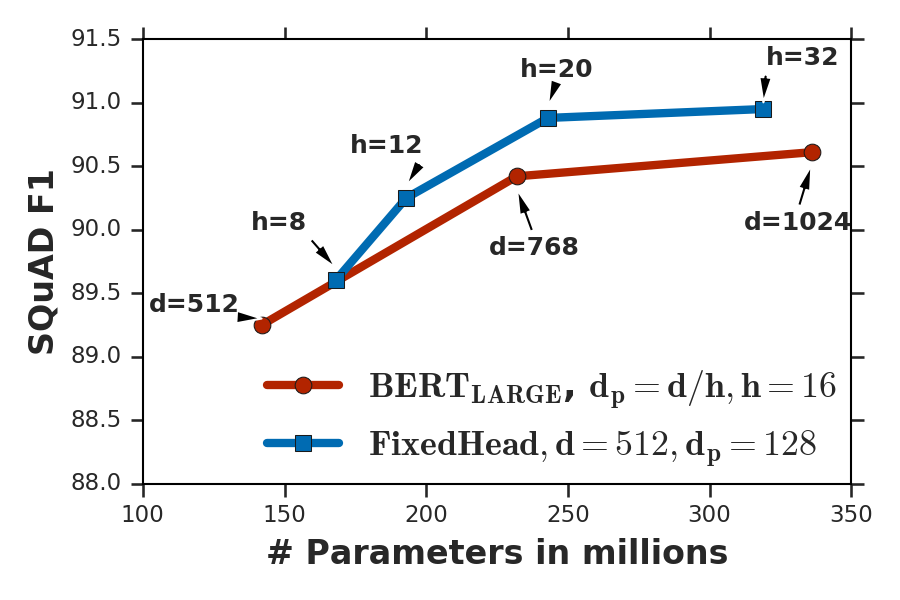}
	\caption{SQuAD F1}
	\label{fig:squad_f1}
	\end{subfigure}
	\begin{subfigure}[b]{0.32\textwidth}
	\includegraphics[width=\textwidth]{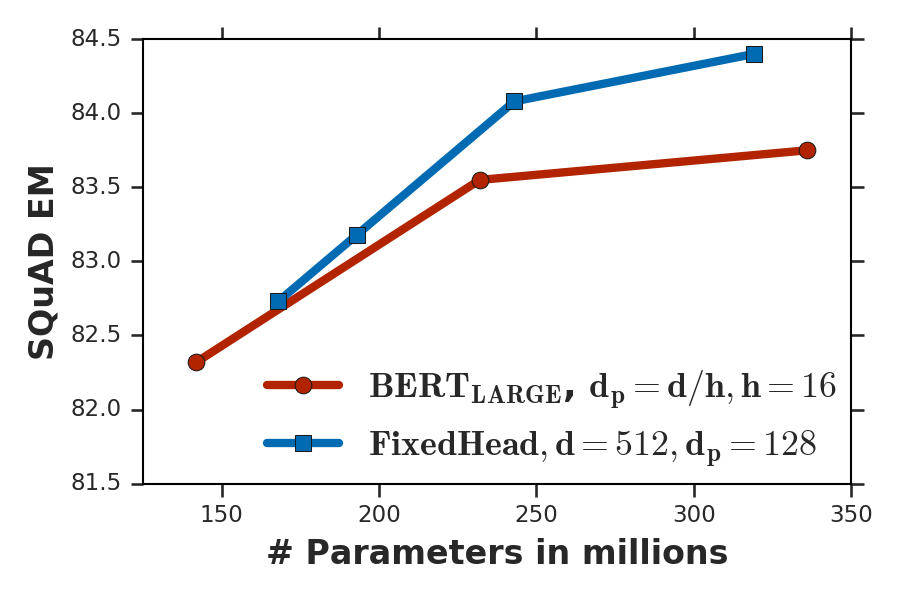}
	\caption{SQuAD EM}
	\label{fig:squad_em}
	\end{subfigure}
	\begin{subfigure}[b]{0.32\textwidth}
	\includegraphics[width=\textwidth]{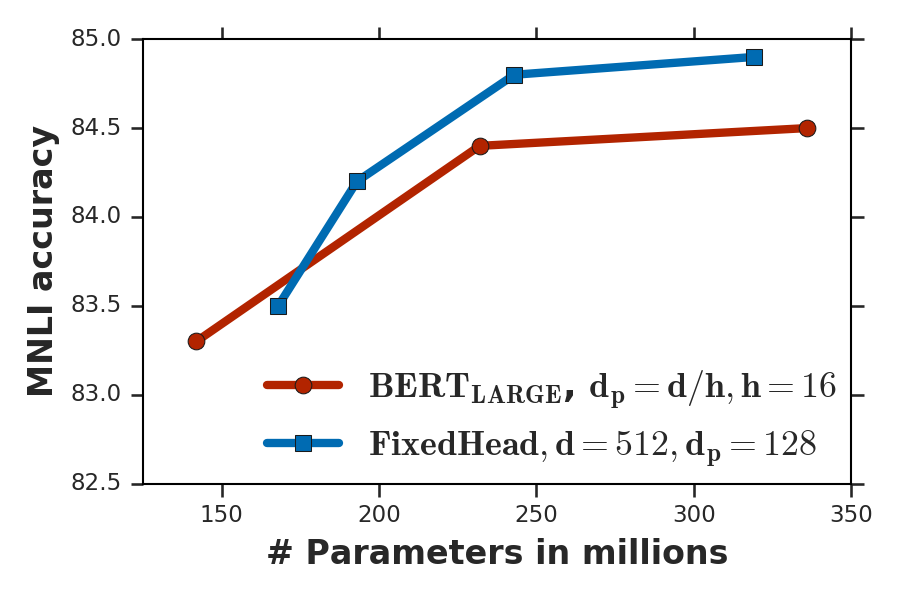}
	\caption{MNLI}
	\label{fig:mnli}
	\end{subfigure}
	\caption{ \CHANGES{Comparison of 24 layer Transformer models trained with the prevalent head size heuristic $\BL$ (baseline) vs.\ the fixed head size model on the SQuAD and MNLI dev sets. We vary the embedding size of the baseline models from 512 to 1024. We train the fixed head size models with a fixed embedding size of 512 and a head size of 128, with a varying number of heads from 8 to 32, while matching the number of parameters. Fixing the head size allows us to train models with a smaller embedding size of 512 and with a better performance.}}
 \label{fig:bert}
\end{figure*}

\begin{figure*}[!t]
\centering
    \begin{subfigure}[b]{0.45\textwidth}
	\includegraphics[width=\textwidth]{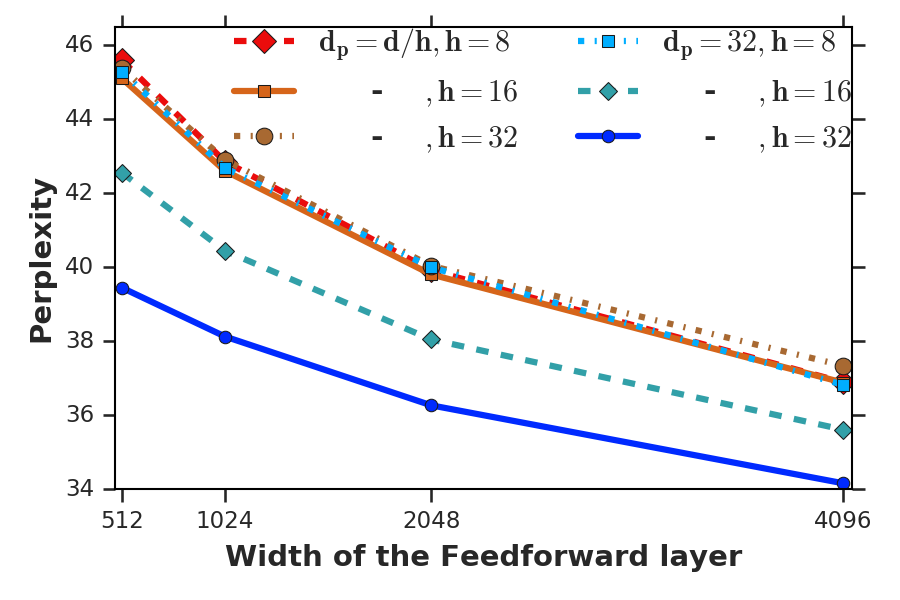}
	\caption{ }
	\label{fig:varying_ff}
	\end{subfigure}
	\begin{subfigure}[b]{0.46\textwidth}
	\includegraphics[width=\textwidth]{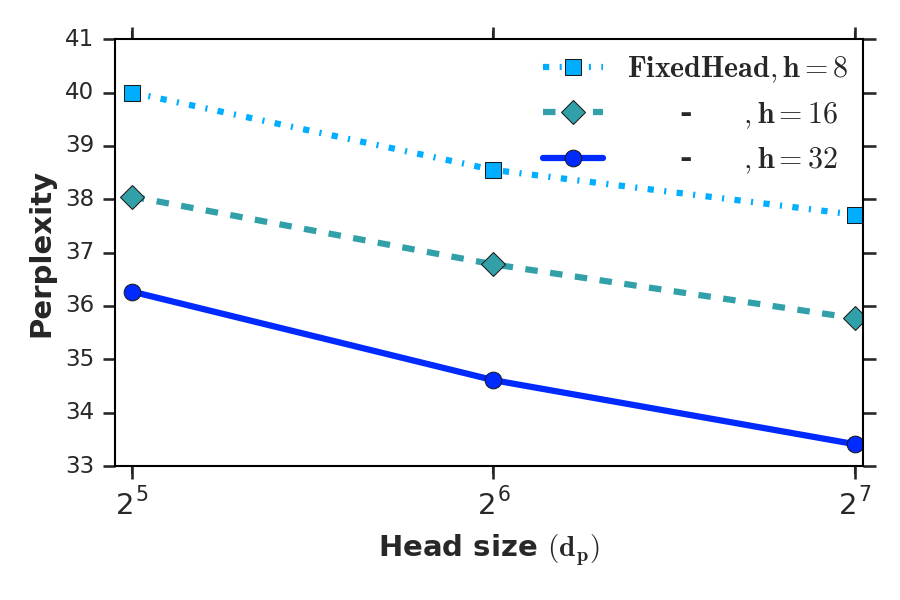}
	\caption{ }
	\label{fig:head_size}
	\end{subfigure}
	\caption{ Ablation studies on LM1B: (a) We fix the embedding size of all the models to 256 and vary the capacity of Transformers trained with the prevalent head size heuristic (baseline) by increasing the size of the feedforward layers. For the fixed head size models we fix the head size to 32, so 8 head fixed head size model is the same as the 8 head baseline model. We notice that again with the standard heuristic increasing the number of heads beyond 16 hurts the performance, whereas with a fixed head size increasing the number of heads monotonically improves the performance. (b) We show the effect of head size on the performance with different number of heads. Both plots clearly show the advantage in having an additional way to tune the capacity of Transformers with a fixed embedding size.}
	\label{fig:lm_ablation}
\end{figure*}

\section{Experiments}\label{sec:experiments}

\CHANGES{The goal of this section is to show that setting the head size in a principled way leads to better performance than using the prevalent heuristic. We again note that while this is a simple hyper-parameter change to the Transformer, setting this to input sequence length as shown in our analysis, allows us to train better models with a smaller embedding size. }

In this section we present our experiments on three standard NLP tasks, language modeling (LM1B), question answering (SQuAD), and sentence entailment (MNLI), to demonstrate: 1) Increasing the number of heads in Transformers beyond a certain point hurts the performance with the prevalent head size heuristic, but always helps with the fixed head size attention layers; 2) Decoupling the head size from embedding size allows us to train models with a smaller embedding size; and 3) Setting the head size appropriately in the Transformers allows us to train models with a better performance scaling. We first describe our experimental setup followed by our results and ablation studies on the proposed modifications.

\subsection{Setup and Datasets}

For the language modeling task we use the one billion word benchmark dataset (LM1B) \citep{lm1b}. This dataset has around 30M training examples and around 300k examples in the test set. We use a sub-word tokenizer with 32k vocab and cap the input to 256 sequence length. We train a 6 layer Transformer model with the ADAM optimizer using the tensor2tensor library \citep{tensor2tensor}. The detailed experimental setting is presented in Section \ref{sec:appx_parameters}.

Multi-Genre Natural Language Inference (MNLI) is a sentence level entailment task, designed to test natural language understanding \citep{MNLI}. Given a premise sentence and a hypothesis sentence, the goal is to predict whether hypothesis entails, contradicts or is neutral to the premise. We report the classification accuracy for this task. Stanford Question Answering Dataset (SQuAD) is a question answering dataset, where given a paragraph and a question, the goal is to predict the sequence of words in the paragraph that constitute the answer to the question \citep{rajpurkar2016squad}. This is a harder word level task, compared to the sentence classification task. We report both Exact Match (EM) and F1 scores for this task. All results in this section are reported on the Dev set, which has not been used in any experimental choices in this paper. 

For these latter two tasks, we follow the two stage approach of first pre-training on a language modeling task and then fine-tuning the models on the task data. We follow the same experimental setup for both pre-training and fine-tuning as BERT \citep{devlin2018bert}, and use their codebase\footnote{https://github.com/google-research/bert}. We first pre-train our models using the masked language model and the next sentence prediction objectives, and then fine tune the pre-trained model for individual tasks \citep{devlin2018bert}. For pre-training we use English Wikipedia and BooksCorpus dataset \citep{zhu2015aligning}. The input to the models is tokenized using the WordPiece representation with 30000 tokens in the vocabulary. We present the key experiment choices in Section \ref{sec:appx_parameters}, and refer the reader to \citet{devlin2018bert} for a complete description of the setup.

\begin{table*}[!t]
\centering
\begin{tabular}{|c| c c c c|} 
 \hline
 \# heads & 8 & 12 & 16  &  32 \\ [0.5ex] 
 \hline
 \# params & 168M & 193M & 218M  & 319M \\[0.5ex] 
 \hline
 SQuAD - F1 & 89.6$\pm$0.17 & 90.25$\pm$0.21 & 90.43$\pm$0.14  & 90.95$\pm$0.14 \\ [0.5ex]
 SQuAD - EM & 82.73$\pm$0.21 &83.18$\pm$0.24 & 83.59$\pm$0.06 & 84.4$\pm$0.29 \\[0.5ex]
 MNLI & 83.5$\pm$0.2 & 84.2$\pm$0.2 & 83.9$\pm$0.2 & 84.9$\pm$0.2 \\ [1ex] 
 \hline
   \multicolumn{5}{c}{(A) Increasing number of heads} \\
  \multicolumn{5}{c}{}\\
  \hline
 head size & 32 & 64 & 128 & 256 \\ [0.5ex] 
 \hline
  \# params & 130M  & 142M & 168M & 218M \\ [0.5ex] 
 \hline
 SQuAD - F1 & 88.53$\pm$0.06 & 89.51$\pm$0.15 & 89.6$\pm$0.17 & 90.33$\pm$0.23 \\ [0.5ex]
 SQuAD - EM & 81.19$\pm$0.21 & 82.41$\pm$0.32 & 82.73$\pm$0.21  & 83.36$\pm$0.48 \\[0.5ex]
 MNLI & 82.5$\pm$0.1 & 83.4$\pm$0.3 &  83.5$\pm$0.2 & 83.9$\pm$0.2 \\ [1ex] 
 \hline
    \multicolumn{5}{c}{(B) Increasing head size} \\
\end{tabular}
\caption{Ablation studies on SQuAD and MNLI: (A) 24 layer Transformer with a fixed head size of 128 and 512 embedding size shows an improvement in the accuracy with the increasing number of heads. (B) The fixed head size model with 512 embedding size and 8 heads shows an improvement in accuracy with the increasing head size. This shows that indeed head size is an important capacity controlling parameter in the self attention architecture.}
\label{table:3}
\end{table*}

\noindent{\bf Choice of the head size.}~Our proposed modification introduces head size $d_p$ as a new model hyper-parameter. We choose head size to be $128$ for our BERT experiments, as most of the pre-training is done with 128 sequence length data. While we have ablation studies (cf.~Table~\ref{table:3}(B)) showing bigger head size improves the performance, there is a tradeoff between increasing the head size vs number of heads vs layers. We found that having sufficiently large head size, e.g., matching the pre-training sequence length, is better than having a larger embedding size. 

\subsection{Results}
\CHANGES{For our first set of experiments we want to see if Transformers trained with a fixed head size  and a smaller embedding size can match the performance of training with the standard head size heuristic but with a larger embedding size. As a baseline for the language modeling task, we train Transformers with the embedding size increasing from 256 to 512 with different number of heads. We train the fixed head size models with a fixed embedding size of 256 and a head size of 32, with an increasing number of heads from 4 to 70. We notice that Transformers with a fixed head size and an embedding size of 256 have better performance than the baseline models with an embedding size of 512 (see Fig.~\ref{fig:lm1b}). We repeat the similar experiment on the other two tasks, where for baseline we train $\BL$, a 24 layer, 16 head Transformer with the standard head size heuristic, with embedding sizes from 512 to 1024. We compare it with the fixed head size model, with an embedding size of 512 and a head size of 128, with an increasing number of heads from 8 to 32. We again notice that the Transformers trained with a fixed head size and 512 embedding size have better performance than the baseline, $\BL$ (see Fig.~\ref{fig:bert}).}

Note that simply trying to increase the head size of the Transformers by decreasing the number of heads does not improve the performance, as decreasing the number of heads reduces the expressive power of the model (see Fig.~\ref{fig:varying_scale_base} in the Appendix). Hence, both the head size and the number of heads needs to be set high enough for better performance.

\subsection{Ablation}

\noindent{\bf Increasing heads.}~From Table~\ref{table:1} and Fig.~\ref{fig:lm1b_params} we can see that increasing the number of heads hurts the performance of the Transformer after a certain number. We repeat the same experiments with the fixed head size Transformer, and present the results in Table~\ref{table:3}(A) and Fig.~\ref{fig:varying_ff}. The results show that the performance of the modified model improves monotonically as the number of heads increase. This is because the model capacity (a function of the head size) is no longer reduced with the increasing number of heads. 

\noindent{\bf Increasing head size.}~In Table~\ref{table:3}(B) and Fig.~\ref{fig:head_size}, we present comparisons between models with different head sizes. This shows that the gains in the performance of the fixed head size models indeed come from adjusting the head size of the query, key and value layers in the attention unit. The table shows a clear trend of better performance with a larger head size, suggesting that it indeed is an important factor in the performance of the attention models.

\section{Conclusion}
In this paper we studied the representation power of the multi-head self attention models and proved the low-rank bottleneck that results from a small head size in the multi-head attention. We showed that the larger embedding size used in the current models is a consequence of this low-rank bottleneck in multi-head attention layers. We propose to instead use fixed head size attention units, with the head size set to input sequence length, to avoid this bottleneck. We showed that it allows us to increase the number of heads without increasing the embedding size. As a consequence we are able to train Transformers with a smaller embedding size and fewer parameters, with better performance. In the future, it will be interesting to experiment with varying head sizes within an attention block and across layers. This requires further understanding of the role of each layer in computing the context, which is an interesting direction for the future work.

\bibliographystyle{abbrvnat}
 \bibliography{references}

 \clearpage
 \appendix
 \section{Notation}
 \begin{center}
\begin{tabular}{ |c|c| } 
 \hline
 Embedding size & $d$ \\ \hline
 Number of layers & $l$ \\ \hline
 Number of heads & $h$ \\  \hline
 Sequence length & $n$ \\  \hline
 Vocab size & $v$ \\  \hline
 Head size & $d_p$ \\  \hline
\end{tabular}
\end{center}

\newcommand{\veczero}{\mathbf{0}}
\newcommand{\vecone}{\mathbf{1}}

\section{Proofs}

\begin{proof}[Proof of Theorem~\ref{thm:comparison}]
First let us rewrite the MultiHead and FixedMultiHead layers as follows. The MultiHead layer can be rewritten as
\begin{equation*}
    f_\vecW (\vecX) = \vecW_o \cdot \text{MultiHead}(\vecX) = \sum_{i=1}^h \vecW_o^i \vecW_v^i \vecX \cdot \sfmx \left[\nicefrac{(\vecW_k^i \vecX)^T (\vecW_q^i \vecX)}{\sqrt{\frac{d}{h}}} \right],
\end{equation*}
where $\vecW_o^i$ are $d \times d/h$ matrices and $\vecW_v^i$, $\vecW_k^i$, and $\vecW_q^i$ are $d/h \times d$ matrices. We denote the collection of all parameter matrices as $\vecW$. 

Similarly, rewrite the fixed head size attention layer as
\begin{equation*}
    g_\vecV (\vecX) = \vecV_o \cdot \text{FixedMultiHead}(\vecX) = \sum_{i=1}^h \vecV_o^i \vecV_v^i \vecX \cdot \sfmx \left[\nicefrac{(\vecV_k^i \vecX)^T (\vecV_q^i \vecX)}{\sqrt{d_p}} \right],
\end{equation*}
where $\vecV_o^i \in \reals^{d \times d_p}$, and $\vecV_v^i, \vecV_k^i, \vecV_q^i \in \reals^{d_p \times d}$. Let $\vecV$ be the collection of all these matrices.

The outline of the proof is basically a case analysis: we divide possible values of $\vecW$ into three categories, and show in each case that there exists a $\vecX$ such that $f_{\vecW}(\vecX) \neq g_{\vecV}(\vecX)$. Here are the three cases:
\begin{itemize}
\item \textbf{Case 1}: $\sum_{i=1}^h \vecW_o^i \vecW_v^i \neq \sum_{i=1}^h \vecV_o^i \vecV_v^i$.
\item \textbf{Case 2}: $\sum_{i=1}^h \vecW_o^i \vecW_v^i = \sum_{i=1}^h \vecV_o^i \vecV_v^i$, and there exists $i \in \{1, \dots, h\}$ such that $\vecU/\sqrt{d_p}-(\vecW_k^i)^T(\vecW_q^i)/\sqrt{d/h}$ is not skew-symmetric.
\item \textbf{Case 3}: $\sum_{i=1}^h \vecW_o^i \vecW_v^i = \sum_{i=1}^h \vecV_o^i \vecV_v^i$, and all $\vecU/\sqrt{d_p}-(\vecW_k^i)^T(\vecW_q^i)/\sqrt{d/h}$ are skew-symmetric.
\end{itemize}

\paragraph{Case 1.} In the first case, we can choose any $\vecv$ such that $(\sum_{i=1}^h \vecW_o^i \vecW_v^i - \sum_{i=1}^h \vecV_o^i \vecV_v^i) \vecv \neq \veczero$. Choose $\vecX = \vecv \vecone^T = \begin{bmatrix} \vecv & \vecv & \dots & \vecv \end{bmatrix}$. Then, note that for any column stochastic matrix $\vecP$, we have $\vecX \vecP = \vecX$. Therefore,
\begin{align*}
    &\sum_{i=1}^h \vecW_o^i \vecW_v^i \vecX \cdot \sfmx \left[ \nicefrac{(\vecW_k^i \vecX)^T (\vecW_q^i \vecX)}{\sqrt{d/h}}  \right] - \sum_{i=1}^h \vecV_o^i \vecV_v^i \vecX \cdot \sfmx \left[\nicefrac{(\vecV_k^i \vecX)^T (\vecV_q^i \vecX)}{\sqrt{d_p}} \right]\\
    = &\sum_{i=1}^h \vecW_o^i \vecW_v^i \vecX - \sum_{i=1}^h \vecV_o^i \vecV_v^i \vecX
    = (\sum_{i=1}^h \vecW_o^i \vecW_v^i - \sum_{i=1}^h \vecV_o^i \vecV_v^i) \vecv \vecone^T \neq \veczero.
\end{align*}

\paragraph{Case 2.} In cases where $\sum_{i=1}^h \vecW_o^i \vecW_v^i = \sum_{i=1}^h \vecV_o^i \vecV_v^i$, since $\sum_{i=1}^h \vecV_o^i \vecV_v^i$ is full rank by assumption and each $\vecW_o^i \vecW_v^i$ is at most rank $d/h$, it follows that all columns in $\vecW_o^i \in \reals^{d \times d/h}$ must be linearly independent. Therefore, for any $\vecv \neq \veczero$, $\{\vecW_o^i \vecW_v^i \vecv, i=1, \dots, h\}$ is a set of linearly independent vectors, because each $\vecW_o^i \vecW_v^i \vecv$ is a linear combination of $d/h$ column vectors of $\vecW_o^i$ that are linearly independent of other column vectors in $\vecW_o^j$, $j \neq i$.

Now consider any $\vecv \in \reals^d$, and $\vecX = \vecv \vece_1^T$, where $\vece_1 = (1, 0, \dots, 0) \in \reals^n$. Define $\phi(t) = \exp(t)/(\exp(t) + n-1)$. Then, we have
\begin{align*}
&
g_{\vecV} (\vecX) =
\sum_{i=1}^h \vecV_o^i \vecV_v^i \vecX \cdot \sfmx \left[ \nicefrac{\vecX^T \vecU \vecX}{\sqrt{d_p}} \right]
= \sum_{i=1}^h \vecV_o^i \vecV_v^i \vecX \cdot \sfmx 
\begin{bmatrix}
\frac{\vecv^T \vecU \vecv}{\sqrt{d_p}} & 0 & \dots & 0\\
0 & 0 & \dots & 0\\
\vdots & \vdots & \ddots & \vdots\\
0 & 0 & \dots & 0
\end{bmatrix}\\
=& \left( \sum_{i=1}^h \vecV_o^i \vecV_v^i \right)
\begin{bmatrix}
\phi \left (\frac{\vecv^T \vecU \vecv}{\sqrt{d_p}} \right) \vecv & \frac{\vecv}{n} & \dots & \frac{\vecv}{n}
\end{bmatrix}
= \left( \sum_{i=1}^h \vecW_o^i \vecW_v^i \right) 
\begin{bmatrix}
\phi \left (\frac{\vecv^T \vecU \vecv}{\sqrt{d_p}} \right) \vecv & \frac{\vecv}{n} & \dots & \frac{\vecv}{n}
\end{bmatrix}.
\end{align*}
Similarly, we can calculate
\begin{align*}
f_{\vecW} (\vecX) &=
    \sum_{i=1}^h \vecW_o^i \vecW_v^i \vecX \cdot \sfmx \left[ \nicefrac{(\vecW_k^i \vecX)^T (\vecW_q^i \vecX)}{\sqrt{d/h}}  \right] \\
    &=
    \sum_{i=1}^h \vecW_o^i \vecW_v^i 
    \begin{bmatrix}
    \phi \left (\frac{\vecv^T (\vecW_k^i)^T \vecW_q^i \vecv}{\sqrt{d/h}} \right) \vecv & \frac{\vecv}{n} & \dots & \frac{\vecv}{n}
    \end{bmatrix}.
\end{align*}
Notice that all the columns of $f_{\vecW} (\vecX)$ and $g_{\vecV} (\vecX)$, from the second columns to the last ones, are the same. We now compare the first columns:
\begin{align*}
    f_{\vecW} (\vecX)_{:, 1} - g_{\vecV} (\vecX)_{:, 1}
    = \sum_{i=1}^h \left ( \phi \left (\frac{\vecv^T (\vecW_k^i)^T \vecW_q^i \vecv}{\sqrt{d/h}} \right) - \phi \left (\frac{\vecv^T \vecU \vecv}{\sqrt{d_p}} \right) \right )\vecW_o^i \vecW_v^i \vecv.
\end{align*}
Recall that for any $\vecv \neq \veczero$, $\vecW_o^i \vecW_v^i \vecv$ are linearly independent, so $f_{\vecW} (\vecX)_{:, 1} - g_{\vecV} (\vecX)_{:, 1} = \veczero$ if and only if all $\phi \left (\frac{\vecv^T (\vecW_k^i)^T \vecW_q^i \vecv}{\sqrt{d/h}} \right) - \phi \left (\frac{\vecv^T \vecU \vecv}{\sqrt{d_p}} \right )$ are zero. However, since there exists $i \in \{1, \dots, h\}$ such that $\vecU/\sqrt{d_p}-(\vecW_k^i)^T(\vecW_q^i)/\sqrt{d/h}$ is not skew-symmetric, we can choose $\vecv$ to be one that satisfies $\frac{\vecv^T (\vecW_k^i)^T \vecW_q^i \vecv}{\sqrt{d/h}} \neq \frac{\vecv^T \vecU \vecv}{\sqrt{d_p}}$, hence making $\phi \left (\frac{\vecv^T (\vecW_k^i)^T \vecW_q^i \vecv}{\sqrt{d/h}} \right) - \phi \left (\frac{\vecv^T \vecU \vecv}{\sqrt{d_p}} \right ) \neq 0$, therefore $f_{\vecW} (\vecX)_{:, 1} - g_{\vecV} (\vecX)_{:, 1} \neq \veczero$.

\paragraph{Case 3.}
Now consider any $\vecX = \begin{bmatrix}\vecv_1 & \vecv_2 & \veczero & \dots &\veczero \end{bmatrix}$, where $\vecv_1$ and $\vecv_2$ will be chosen later.
Define $\phi_1(t_1, t_2) = \exp(t_1)/(\exp(t_1) + \exp(t_2) + n-2)$, $\phi_2(t_1, t_2) = \exp(t_2)/(\exp(t_1) + \exp(t_2) + n-2)$. Then, we have
\begin{align*}
&
g_{\vecV} (\vecX) 
= \sum_{i=1}^h \vecV_o^i \vecV_v^i \vecX \cdot \sfmx 
\begin{bmatrix}
\frac{\vecv_1^T \vecU \vecv_1}{\sqrt{d_p}} & \frac{\vecv_1^T \vecU \vecv_2}{\sqrt{d_p}} & 0 & \dots & 0\\
\frac{\vecv_2^T \vecU \vecv_1}{\sqrt{d_p}} & \frac{\vecv_2^T \vecU \vecv_2}{\sqrt{d_p}} & 0 & \dots & 0\\
0 & 0 & 0 & \dots & 0\\
\vdots & \vdots & \vdots & \ddots & \vdots\\
0 & 0 & 0 & \dots & 0
\end{bmatrix}.
\end{align*}
Therefore, the first column of $g_\vecV(\vecX)$ can be written as
\begin{align*}
g_{\vecV} (\vecX)_{:,1}
=
\left( \sum_{i=1}^h \vecW_o^i \vecW_v^i \right) 
\left [
\phi_1 \left (\frac{\vecv_1^T \vecU \vecv_1}{\sqrt{d_p}}, \frac{\vecv_2^T \vecU \vecv_1}{\sqrt{d_p}} \right) \vecv_1 + 
\phi_2 \left (\frac{\vecv_1^T \vecU \vecv_1}{\sqrt{d_p}}, \frac{\vecv_2^T \vecU \vecv_1}{\sqrt{d_p}} \right) \vecv_2
\right ].
\end{align*}
Similarly, the first column of $f_{\vecW}(\vecX)$ is
\begin{align*}
f_{\vecW} (\vecX)_{:,1}
=
\sum_{i=1}^h \vecW_o^i \vecW_v^i
\Bigg [ &
\phi_1 \left (\frac{\vecv_1^T (\vecW_k^i)^T \vecW_q^i \vecv_1}{\sqrt{d/h}}, \frac{\vecv_2^T (\vecW_k^i)^T \vecW_q^i \vecv_1}{\sqrt{d/h}} \right) \vecv_1 + \\
&
\phi_2 \left (\frac{\vecv_1^T (\vecW_k^i)^T \vecW_q^i \vecv_1}{\sqrt{d/h}}, \frac{\vecv_2^T (\vecW_k^i)^T \vecW_q^i \vecv_1}{\sqrt{d/h}} \right) \vecv_2
\Bigg ].
\end{align*}
Since $\vecU/\sqrt{d_p}-(\vecW_k^1)^T(\vecW_q^1)/\sqrt{d/h}$ is skew-symmetric by assumption, we have $\vecv_1^T \left (\frac{\vecU}{\sqrt{d_p}}-\frac{(\vecW_k^1)^T(\vecW_q^1)}{\sqrt{d/h}} \right )\vecv_1 = 0$ for all $\vecv_1$. Recall that $\vecU$ is rank-$d_p$ by assumption, so $\vecU/\sqrt{d_p}-(\vecW_k^1)^T(\vecW_q^1)/\sqrt{d/h}$ is at least rank $d_p - d/h \geq 1$, so we can choose any $\vecv_1$ such that $\left (\frac{\vecU}{\sqrt{d_p}}-\frac{(\vecW_k^1)^T(\vecW_q^1)}{\sqrt{d/h}} \right )\vecv_1 \neq \veczero$. 

If both $\frac{\vecU}{\sqrt{d_p}} \vecv_1$ and $\frac{(\vecW_k^1)^T(\vecW_q^1)}{\sqrt{d/h}} \vecv_1$ are nonzero,
We can always choose $\tilde \vecv_2$ such that $\tilde \vecv_2 ^T \left (\frac{\vecU}{\sqrt{d_p}}\right )\vecv_1 > 0$ and $\tilde \vecv_2 ^T \left (\frac{(\vecW_k^1)^T(\vecW_q^1)}{\sqrt{d/h}}\right )\vecv_1 < 0$. This means that if we choose $\vecv_2 = \alpha \tilde \vecv_2$ and scale $\alpha \rightarrow \infty$, 
\begin{align*}
    &\phi_1 \left (\frac{\vecv_1^T \vecU \vecv_1}{\sqrt{d_p}}, \frac{\vecv_2^T \vecU \vecv_1}{\sqrt{d_p}} \right) \rightarrow 0,
    ~~
    \phi_2 \left (\frac{\vecv_1^T \vecU \vecv_1}{\sqrt{d_p}}, \frac{\vecv_2^T \vecU \vecv_1}{\sqrt{d_p}} \right) \rightarrow 1,\\
    &\phi_1 \left (\frac{\vecv_1^T (\vecW_k^1)^T \vecW_q^1 \vecv_1}{\sqrt{d/h}}, \frac{\vecv_2^T (\vecW_k^1)^T \vecW_q^1 \vecv_1}{\sqrt{d/h}} \right) \rightarrow
    \frac{\exp(\vecv_1^T (\vecW_k^1)^T \vecW_q^1 \vecv_1/\sqrt{d/h})}
    {\exp(\vecv_1^T (\vecW_k^1)^T \vecW_q^1 \vecv_1/\sqrt{d/h})+n-2},\\
    &
    \phi_2 \left (\frac{\vecv_1^T (\vecW_k^1)^T \vecW_q^1 \vecv_1}{\sqrt{d/h}}, \frac{\vecv_2^T (\vecW_k^1)^T \vecW_q^1 \vecv_1}{\sqrt{d/h}} \right) \rightarrow 0.
\end{align*}

Then, consider the difference $f_{\vecW} (\vecX)_{:,1} - g_{\vecV} (\vecX)_{:,1}$. Recall that for any $\vecv$, $\vecW_o^1 \vecW_v^1 \vecv$ is independent of $\{\vecW_o^i \vecW_v^i \vecv, i \neq 1 \}$. This means that, to show $f_{\vecW} (\vecX)_{:,1} - g_{\vecV} (\vecX)_{:,1} \neq \veczero$, it suffices to show that
\begin{align*}
&\Bigg [
\phi_1 \left (\frac{\vecv_1^T (\vecW_k^1)^T \vecW_q^1 \vecv_1}{\sqrt{d/h}}, \frac{\vecv_2^T (\vecW_k^1)^T \vecW_q^1 \vecv_1}{\sqrt{d/h}} \right) - 
\phi_1 \left (\frac{\vecv_1^T \vecU \vecv_1}{\sqrt{d_p}}, \frac{\vecv_2^T \vecU \vecv_1}{\sqrt{d_p}} \right)
\Bigg ]
\vecW_o^1 \vecW_v^1 \vecv_1 
 + \\
&\Bigg [
\phi_2 \left (\frac{\vecv_1^T (\vecW_k^1)^T \vecW_q^1 \vecv_1}{\sqrt{d/h}}, \frac{\vecv_2^T (\vecW_k^1)^T \vecW_q^1 \vecv_1}{\sqrt{d/h}} \right) - 
\phi_2 \left (\frac{\vecv_1^T \vecU \vecv_1}{\sqrt{d_p}}, \frac{\vecv_2^T \vecU \vecv_1}{\sqrt{d_p}} \right)
\Bigg ]
\vecW_o^1 \vecW_v^1 \vecv_2
\neq \veczero.
\end{align*}
If we scale $\vecv_2 = \alpha \tilde \vecv_2$ with large enough $\alpha$, the second term will dominate the first term and the first term will never be able to cancel the second one. Thus, by choosing large enough $\alpha > 0$, we can make sure that the sum is nonzero.

Even in case where one of $\frac{\vecU}{\sqrt{d_p}} \vecv_1$ and $\frac{(\vecW_k^1)^T(\vecW_q^1)}{\sqrt{d/h}} \vecv_1$ is zero (say $\frac{(\vecW_k^1)^T(\vecW_q^1)}{\sqrt{d/h}} \vecv_1 = \veczero$), we can choose $\tilde \vecv_2 = \frac{\vecU}{\sqrt{d_p}} \vecv_1$ and use a similar scaling argument. By choosing large enough $\alpha > 0$ and $\vecv_2 = \alpha \tilde \vecv_2$, one can show that the difference $f_{\vecW} (\vecX)_{:,1} - g_{\vecV} (\vecX)_{:,1}$ is nonzero.
\end{proof}
\section{Experimental settings}\label{sec:appx_parameters}

For our experiments with the language modeling (LM1B dataset), we train 6 layer Transformer models. We use a batch size of 4096 and train for 250k steps. We use a learning rate of 0.1 with a linear warm up for the first 10k steps. We decay the learning rate with the square root of the number of steps. We train the baseline models, with the prevalent head size heuristic, with the embedding dimension varying from 256 to 512. We fix the width of the feed forward layer in the Transformer to be 1024. In addition, we use weight decay of 0.01 and dropout with probability of 0.1 on all the layers.

For our experiments with BERT, we follow the same experimental settings as in \citep{devlin2018bert}. We present the  key  details here and refer the reader to \citep{devlin2018bert}. We train with a batch size of 1024 for 450k steps with inputs of sequence length $n$ = 128 followed by 50k steps with inputs of sequence length 512. In contrast the BERT paper uses a batch size of 512, and does the pre-training for 900K steps with 128 sequence length inputs and 100k steps with 512 sequence length inputs. We train using ADAM with a learning rate of 1e-4, and a linear warmup and decay schedule as in BERT. We use 5k warmup steps for the first stage, and a re-warmup of 3k steps for the second stage \citep{you2019reducing}. Again, we use weight decay of 0.01 and dropout with probability of 0.1 on all the layers.

For the language modeling task, training is performed on 4 TPUv2 chips for a couple of hours. For BERT models training is performed on 16 TPUv3 chips in the first stage and 64 TPUv3 chips for the second stage. Pre-training with this configuration takes between 2 to 3 days. We did not attempt to find the optimal hyper-parameters for the fixed head size architecture, and use the same hyper-parameters as used for training the BERT models.

\section{Additional experimental results}\label{sec:appx_experiments}

\begin{figure}[ht]
\centering
	\includegraphics[width=.6\linewidth]{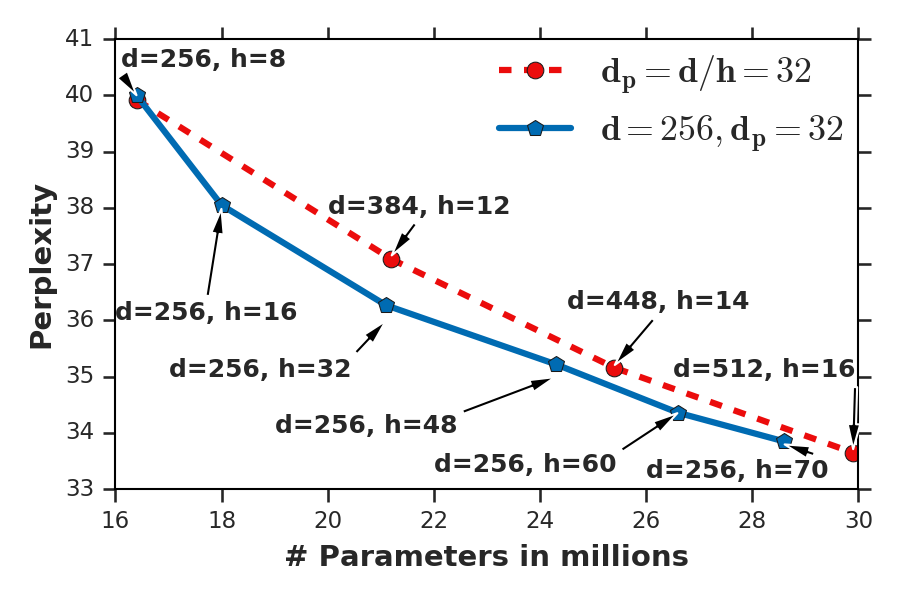}
	\caption{Performance of the Transformers trained with the prevalent head size heuristic (baseline) compared with the fixed head size ($d_p$) models for a language modeling task (LM1B) on the test set. Unlike Fig.\ref{fig:lm1b}, we vary both the embedding size and the number of heads of the baseline models to keep their head size fixed to 32. We train the fixed head size models with a fixed embedding size of 256 and a head size of 32, and vary the number of heads from 4 to 70, while matching the number of parameters. The plot again clearly indicates the advantage of the fixed head size models. The main issue with the baseline models is that fixing the head size to 32 forces the number of heads to be small when the  embedding size is small. Reducing the number of heads below certain threshold hurts the performance of the Transformer.}
	\label{fig:varying_scale_base}
\end{figure}

\begin{table}[!t]
\centering
\begin{tabular}{|c| c c c c |} 
 \hline
 \# heads & 8 & 12 & 16 & 20 \\ [0.5ex] 
 \hline
 \# params & 214M & 252M & 290M & 327M \\[0.5ex] 
 \hline
 SQuAD - F1 & 90.35$\pm$0.14 & 90.48$\pm$0.09 & 90.92$\pm$0.14 & 90.89$\pm$0.08 \\ [0.5ex]
 SQuAD - EM & 83.37$\pm$0.12 & 83.67$\pm$0.03 & 84.16$\pm$0.35 & 84.29$\pm$0.16 \\[0.5ex]
 MNLI  & 84.4$\pm$0.2 & 84.4$\pm$0.2 & 84.7$\pm$0.1 & 85.1$\pm$0.4 \\ [1ex] 
 \hline
   \multicolumn{5}{c}{(A) Increasing number of heads} \\
\end{tabular}
\caption{(A): 24 layer Transformer trained with a fixed head size of 128 and an embedding size of 768 shows an improvement in the accuracy with the increasing number of heads.}
\label{table:appx_3}
\end{table}

\end{document}